\documentclass[11pt]{article}
\usepackage{fullpage}

\usepackage{amsmath,amsfonts,bm}

\def\eqref#1{equation~\ref{#1}}

\def\1{\bm{1}}

\def\rmK{{\mathbf{K}}}

\def\rmX{{\mathbf{X}}}

\def\vx{{\bm{x}}}
\def\vy{{\bm{y}}}

\DeclareMathAlphabet{\mathsfit}{\encodingdefault}{\sfdefault}{m}{sl}
\SetMathAlphabet{\mathsfit}{bold}{\encodingdefault}{\sfdefault}{bx}{n}

\def\gG{{\mathcal{G}}}

\newcommand{\E}{\mathbb{E}}

\newcommand{\R}{\mathbb{R}}

\usepackage{hyperref,natbib}
\usepackage{url}
\usepackage{diagbox}

\usepackage{graphicx} 
\usepackage{caption}
\usepackage{subcaption}
\usepackage{slashbox}
\usepackage{enumerate}
\usepackage{enumitem}

\usepackage{natbib}
\usepackage{amsmath}
\usepackage{amsthm}
\usepackage{amssymb}
\usepackage{tikz}
\usepackage{xcolor}
\usetikzlibrary{arrows}
\usepackage{hyperref}
\hypersetup{colorlinks,
	linkcolor=blue,
	citecolor=blue,
	urlcolor=magenta,
	linktocpage,
	plainpages=false}
\usepackage{booktabs} 
\allowdisplaybreaks

\usepackage{mathrsfs}
\usepackage{bbm}

\usepackage{algorithm}
\usepackage[noend]{algpseudocode}
\usepackage{hyperref}
\usepackage{bm,todonotes}

\usepackage{thmtools}
\usepackage{thm-restate}
\declaretheorem[name=Theorem,numberwithin=section]{thm}

\newcommand{\conv}{*}
\newcommand{\diag}{\mathrm{diag}}

\newcommand{\twotwomat}{\mat{\Lambda}}

\newcommand{\tr}{\mathrm{tr}}
\newcommand{\trace}[1]{\mathrm{tr}\left(#1\right)}

\def\R{\mathbb{R}}

\newcommand{\mat}[1]{\bm{#1}}
\newcommand{\vect}[1]{\bm{#1}}

\DeclareMathOperator*{\expect}{\mathbb{E}}

\newcommand{\params}{\vect{\theta}}

\newcommand{\relu}[1]{\sigma\left(#1\right)}

\newcommand{\act}[1]{\sigma\left(#1\right)}
\newcommand{\deract}[1]{\dot{\sigma}\left(#1\right)}

\newcommand{\valpha}{\bm{\alpha}}
\newcommand{\talpha}{\widetilde{\alpha}}
\newcommand{\vtalpha}{\widetilde{\bm{\alpha}}}

\newcommand{\nnw}{{P}}
\newcommand{\nnh}{{Q}}
\newcommand{\nnc}{{C}}
\newcommand{\gauss}{\mathcal{N}}
\newcommand{\indset}{\mathcal{D}}

\newtheorem{cor}{Corollary}[section]

\newtheorem{defn}{Definition}[section]

\newcommand{\gap}{\textsf{GAP}}
\newcommand{\lap}{\textsf{LAP}}
\newcommand{\trans}{\mathcal{T}}
\newcommand{\flip}{\mathcal{F}}
\newcommand{\cnngp}{\textsf{CNN-GP}}
\newcommand{\cntk}{\textsf{CNTK}}
\newcommand{\ntk}{\textsf{NTK}}
\newcommand{\cnn}{\textsf{CNN}}
\newcommand{\fc}{\textsf{FC}}
\newcommand{\modular}{\mathrm{mod}}
\newcommand{\lapc}{c}
\newcommand{\boxblur}{\textsf{BBlur}}

\newenvironment{itemize*}%
{\begin{itemize}[leftmargin=*,topsep=0pt]%
		\setlength{\itemsep}{0pt}%
		\setlength{\parskip}{0pt}}%
	{\end{itemize}}

\newenvironment{enumerate*}%
{\begin{enumerate}[leftmargin=*,topsep=0pt]%
		\setlength{\itemsep}{0pt}%
		\setlength{\parskip}{0pt}}%
{\end{enumerate}}

\usepackage{pgfplots}
\pgfplotsset{compat=newest}
\title{
Enhanced Convolutional Neural Tangent Kernels\thanks{The first three authors contribute equally.}}
\date{}

\author{Zhiyuan Li\thanks{Princeton University. Email: \texttt{zhiyuanli@cs.princeton.edu}. }
 \and Ruosong Wang\thanks{Carnegie Mellon University. Email: \texttt{ruosongw@andrew.cmu.edu}.}
 \and Dingli Yu\thanks{Princeton University. Email: \texttt{dingliy@cs.princeton.edu}.}
 \and Simon S. Du\thanks{Institute for Advanced Study. Email: \texttt{ssdu@ias.edu}.}
 \and Wei Hu\thanks{Princeton University. Email: \texttt{huwei@cs.princeton.edu}.}
 \and Ruslan Salakhutdinov\thanks{Carnegie Mellon University. Email: \texttt{rsalakhu@cs.cmu.edu}.}
 \and Sanjeev Arora\thanks{Princeton University and Institute for Advanced Study. Email: \texttt{arora@cs.princeton.edu}.}
 }

\begin{document}

\maketitle

\begin{abstract}
Recent research shows that for training with $\ell_2$ loss, convolutional neural networks (\cnn s) whose width (number of channels in convolutional layers) goes to infinity correspond to regression with respect to the \cnn~Gaussian Process kernel (\cnngp) if only the last layer is trained, and  correspond to regression with respect to  the Convolutional Neural Tangent Kernel (\cntk) if all layers are trained. An exact algorithm to compute \cntk~\citep{arora2019exact} yielded the finding that classification accuracy of \cntk~on CIFAR-10 is within $6$-$7\%$ of that of the corresponding \cnn~architecture (best figure being around $78$\%) which is interesting performance for a fixed kernel.

Here we show how to significantly enhance the performance of these kernels using two ideas. (1) Modifying the kernel using a new operation called {\em Local Average Pooling} (\lap) which preserves efficient computability of the kernel and inherits the spirit of   standard data augmentation using pixel shifts. Earlier papers were unable to incorporate naive data augmentation because of the quadratic training cost of kernel regression. This idea is inspired by {\em Global Average Pooling} (\gap), which we show for \cnngp~and \cntk~is equivalent to full translation data augmentation. (2) Representing the input image using a pre-processing technique proposed by~\citet{coates2011analysis}, which uses a single convolutional layer composed of random image patches.

On CIFAR-10, the resulting kernel, \cnngp~with \lap~and horizontal flip data augmentation, achieves $89\%$ accuracy, matching the performance of AlexNet~\citep{krizhevsky2012imagenet}. Note that this is the best such result we know of for a classifier that is not a trained neural network. Similar improvements are obtained for Fashion-MNIST.

\end{abstract}

\section{Introduction}
\label{sec:intro}

Recent research shows that for training with $\ell_2$ loss, convolutional neural networks (\cnn s) whose width (number of channels in convolutional layers) goes to infinity, correspond to regression with respect to the \cnn~Gaussian Process kernel (\cnngp) if only the last layer is trained, and  correspond to regression with respect to  the Convolutional Neural Tangent Kernel (\cntk) if all layers are trained~\citep{jacot2018neural}.
An exact algorithm was given \citep{arora2019exact} to compute \mbox{\cntk}~for \cnn~architectures, as well as those that include a {\em Global Average Pooling} (\gap) layer (defined below). This is a fixed kernel that inherits some benefits of \cnn s, including exploitation of locality via convolution, as well as multiple layers of processing. For CIFAR-10, incorporating \gap~into the kernel improves classification accuracy by up to $10\% $ compared to pure convolutional \cntk.

While this performance is encouraging for a fixed kernel, the best accuracy is still under $78\%$, which is disappointing even compared to AlexNet. One hope for improving the accuracy further is to somehow capture modern innovations such as batch normalization, data augmentation, residual layers, etc. in \cntk. The current paper shows how to incorporate simple data augmentation. Specifically, the idea of creating new training images from existing images using pixel translation and flips, while assuming that these operations should not change the label. Since deep learning uses stochastic gradient descent (SGD), it is trivial to do such data augmentation on the fly. However, it's unclear  how to efficiently incorporate data augmentation in kernel regression, since training time is quadratic in the number of training images.


Thus somehow data augmentation has to be incorporated into the computation of the kernel itself.  The main observation here is that the above-mentioned algorithm for computing \cntk~involves a dynamic programming whose recursion depth is equal to the depth of the corresponding finite \mbox{\cnn}. It is possible to impose symmetry constraints at any desired layer during this computation.
 In this viewpoint, it can be shown that prediction using \cntk/\cnngp~with \gap~is \mbox{equivalent} to prediction using \cntk/\cnngp~without \gap~but with \emph{full translation data augmentation} with wrap-around at the boundary.
The translation invariance property implicitly assumed in data augmentation is exactly equivalent to an imposed symmetry constraint in the computation of the \cntk~which in turn is derived from the pooling layer in the \cnn.  See Section~\ref{sec:gap} for more details.

Thus \gap~corresponds to full translation data augmentation scheme, but in practice
such data augmentation creates unrealistic images (cf. Figure~\ref{fig:gap_lap}) and training on them can harm performance.
However, the idea of incorporating symmetry in the dynamic programming leads to a variant we call \emph{Local Average Pooling} (\lap). This implicitly is like data  augmentation where image labels are assumed to be invariant to small translation, say by a few pixels.
This operation also suggests a new pooling layer for \cnn s which we call $\boxblur$ and also find it beneficial for \cnn s in  experiments. 

Experimentally, we find \lap~significantly enhances the performance as discussed below.
\begin{itemize}
	\item 	In extensive experiments on CIFAR-10 and Fashion-MNIST, we find that \lap~consistently improves performance of  \cnngp~and \cntk.
	In particular, we find \cnngp~with \lap~achieves $81\%$ on CIFAR-10 dataset, outperforming the best previous kernel predictor by $3\%$.
	\item When using the technique proposed by~\citet{coates2011analysis}, which uses  randomly sampled patches from training data as filters to do pre-processing,\footnote{See Section~\ref{sec:exp_cn} for the precise procedure.} \cnngp~with \lap~and horizontal flip data augmentation achieves $89\%$  accuracy on CIFAR-10, matching the performance of AlexNet~\citep{krizhevsky2012imagenet} and is the strongest classifier that is not a trained neural network.\footnote{
		https://benchmarks.ai/cifar-10
		}
\item We also derive a layer for \cnn~that corresponds to \lap~and observe that it improves the performance on certain architectures.
\end{itemize}

\section{Related Work}
\label{sec:rel}
Data augmentation has long been known to improve the performance of neural networks and kernel methods~\citep{sietsma1991creating,scholkopf1996incorporating}.
Theoretical study of data augmentation dates back to \cite{chapelle2001vicinal}.
Recently, \citet{dao2018kernel} proposed a theoretical framework for understanding data augmentation and showed data augmentation with a kernel classifier can have  feature averaging and variance regularization effects.
More recently, \citet{chen2019invariance} quantitatively shows in certain settings, data augmentation provably improves the classifier performance.
For more comprehensive discussion on data augmentation and its properties, we refer readers to \cite{dao2018kernel,chen2019invariance} and references therein.

\cnngp~and \cntk~correspond to infinitely wide \cnn~with different training strategies (only training the top layer or training all layers jointly).
The correspondence between infinite neural networks and kernel machines was first noted by~\citet{neal1996priors}.
More recently, this was extended to deep and convolutional neural networks~\citep{lee2018deep,matthews2018gaussian,novak2019bayesian,garriga-alonso2018deep}.
These kernels correspond to neural networks where only the last layer is trained.
A recent line of work studied overparameterized neural networks where all layers are trained~\citep{allen2018convergence,du2018provably,du2018global,li2018learning,zou2018stochastic}.
Their proofs imply the gradient kernel is close to a fixed kernel which only depends the training data and the neural network architecture.
These kernels thus correspond to neural networks where are all layers are trained.
\citet{jacot2018neural} named this kernel neural tangent kernel (\ntk).
\citet{arora2019exact} formally proved polynomially wide neural net predictor trained by gradient descent is equivalent to \ntk~predictor.
Recently, \ntk s induced by various neural network architectures are derived and shown to achieve strong empirical performance~\citep{arora2019exact,yang2019scaling,du2019graph}.

Global Average Pooling (\gap) is first proposed in~\citet{lin2013network} and is common in modern \cnn~design~\citep{springenberg2014striving, he2016deep, huang2017densely}.
However, current theoretical understanding on \gap~is still rather limited.
It has been conjectured in ~\citet{lin2013network} that \gap~reduces the number of parameters in the last fully-connected layer and thus avoids overfitting, and \gap~is more robust to spatial translations of the input since it sums out the spatial information.
In this work, we study \gap~from the \cnngp~and \cntk~perspective, and draw an interesting connection between \gap~and data augmentation.

The approach proposed in~\cite{coates2011analysis} is one of the best-performing approaches on CIFAR-10 preceding modern CNNs.
In this work we combine \cntk~with \lap~and the approach in~\cite{coates2011analysis} to achieve the best performance for classifiers that are not trained neural networks.

\vspace{-0.3cm}
\section{Preliminaries}
\label{sec:pre}
\subsection{Notation}
\label{sec:notations}
We use bold-faced letters for vectors, matrices and tensors.
For a vector $\vect a$, we use $[\vect a]_i$ to denote its $i$-th entry.
For a matrix $\mat A$, we use $\left[\mat A\right]_{i,j}$ to denote its $(i, j)$-th entry.
For an order $4$ tensor $\mat{T}$, we use $\mat \left[\mat T\right]_{i,j,i',j'}$ to denote its $(i,j,i',j')$-th entry.
For an order $4$ tensor, wet use $\tr\left(\mat{T}\right)$ to denote $\sum_{i,j}\mat{T}_{i,j,i,j}$.
For an order $d$ tensor $\mat{T} \in \mathbb{R}^{C_1 \times C_2 \times \ldots \times C_d}$ and an integer $\alpha \in [C_d]$, we use $\mat{T}_{(\alpha)} \in \mathbb{R}^{C_1 \times C_2 \times \ldots \times C_{d - 1}}$ to denote the order $d - 1$ tensor formed by fixing the coordinate of the last dimension to be $\alpha$.

\subsection{\cnn, \cnngp~and \cntk}
\label{sec:cnngp_cntk}
In this section we give formal definitions of \cnn, \cnngp~and \cntk~that we study in this paper.
Throughout the paper, we let $\nnw$ be the width and $\nnh$ be the height of the image.
We use $q \in \mathbb{Z}_+$ to denote the filter size. In practice, $q=1$, $3$, $5$ or $7$.

\paragraph{Padding Schemes.}
In the definition of \cnn, \cntk~and \cnngp, we may use different padding schemes.
Let $\vect{x} \in \mathbb{R}^{\nnw \times \nnh}$ be an image.
For a given index pair $(i, j)$ with $i \le 0$, $i \ge \nnw+1 $, $j \le 0$ or $j \ge \nnh+1$, different padding schemes define different value for $\left[\mat{x}\right]_{i,j}$.
For {\em circular padding}, we define $\left[\mat{x}\right]_{i,j}$ to be $\left[\mat{x}\right]_{i~\modular~\nnw, j~\modular~\nnh}$.
For {\em zero padding}, we simply define $\left[\mat{x}\right]_{i,j}$ to be $0$.
Note the difference between circular padding and zero padding occurs only on the boundary of images.
We will prove our theoretical results for the circular padding scheme to avoid boundary effects.

\paragraph{\cnn.} Now we describe \cnn~with and without \gap.
For any input image $\vect{x}$, after $L$ intermediate layers, we obtain $\vect{x}^{(L)} \in \mathbb{R}^{\nnw \times \nnh \times \nnc^{(L)}}$ where $\nnc^{(L)}$ is the number of channels of the last layer.
See Section~\ref{sec:cnngp_cntk_def} for the definition of $\vect{x}^{(L)}$.
For the output, there are two choices: with and without \gap.
\begin{itemize}
\item Without \gap: the final output is defined as
	\[
	f(\params,\vect{x}) = \sum_{\alpha=1}^{\nnc^{(L)}} \left\langle \mat{W}_{(\alpha)}^{(L+1)},\vect{x}_{(\alpha)}^{(L)}\right\rangle
	\]
	where $\vect{x}^{(L)}_{(\alpha)} \in \mathbb{R}^{\nnw \times \nnh}$, and $\mat{W}_{(\alpha)}^{(L+1)} \in \mathbb{R}^{\nnw \times \nnh}$ is the weight of the last fully-connected layer.
	
	\item With \gap: the final output is defined as
	\[
	f(\params,\vect{x}) = \frac{1}{PQ} \sum_{\alpha=1}^{\nnc^{(L)}}
	\mat{W}_{(\alpha)}^{(L+1)} \cdot \sum_{(i,j) \in [\nnw] \times [\nnh]} \left[\vect{x}_{(\alpha)}^{(L)}\right]_{i,j}
	\]
	where $\mat{W}_{(\alpha)}^{(L+1)} \in \mathbb{R}$ is the weight of the last fully-connected layer.
\end{itemize}

\paragraph{\cnngp~and \cntk.} Now we describe \cnngp~and \cntk.
Let $\vect{x},\vect{x}'$ be two input images.
We denote the $L$-th layer's \cnngp~kernel as $\mat{\Sigma}^{(L)}\left(\vect{x},\vect{x}'\right) \in \mathbb{R}^{[\nnw] \times [\nnh] \times [\nnw] \times [\nnh]}$ and  the $L$-th layer's \cntk~kernel as $\mat{\Theta}^{(L)}\left(\vect{x},\vect{x}'\right) \in \mathbb{R}^{[\nnw] \times [\nnh] \times [\nnw] \times [\nnh]}$.
See Section~\ref{sec:cnngp_cntk_def} for the precise definitions of $\mat{\Sigma}^{(L)}\left(\vect{x},\vect{x}'\right)$ and $\mat{\Theta}^{(L)}\left(\vect{x},\vect{x}'\right)$.
For the output kernel value, again, there are two choices, without \gap~(equivalent to using a fully-connected layer) or with \gap.
\begin{itemize}
\item Without \gap: the output of \cnngp~is \[ \mat{\Sigma}_{\fc}\left(\vect{x},\vect{x}'\right) =\trace{\mat{\Sigma}^{(L)}(\vect{x},\vect{x}')}\] and the output of \cntk~is \[\mat\Theta_{\fc}\left(\vect{x},\vect{x}'\right) = \trace{\mat{\Theta}^{(L)}(\vect{x},\vect{x}')}.\]
\item With \gap: the output of \cnngp~is \[\mat{\Sigma}_{\gap}\left(\vect{x},\vect{x}'\right)=\frac{1}{\nnw^2 \nnh^2}\sum\nolimits_{i,j,i',j' \in [\nnw]\times [\nnh]\times [\nnw]\times [\nnh]}\left[\mat{\Sigma}^{(L)}\left(\vect{x},\vect{x}'\right)\right]_{i,j,i',j'}\] and the output of \cntk~is \[\mat\Theta_{\gap}\left(\vect{x},\vect{x}'\right)= \frac{1}{\nnw^2 \nnh^2}\sum\nolimits_{i,j,i',j' \in [\nnw]\times [\nnh]\times [\nnw]\times [\nnh]}\left[\mat{\Theta}^{(L)}\left(\vect{x},\vect{x}'\right)\right]_{i,j,i',j'}.\]
\end{itemize}

\paragraph{Kernel Prediction.} Lastly, we recall the formula for kernel regression.
For simplicity, throughout the paper, we will assume all kernels are invertible.
Given a kernel $\rmK \left(\vect{x},\vect{x}'\right)$ and a dataset $(\mat{X}, \vect{y})$ with data $\left\{\left(\vect{x}_i,y_i\right)\right\}_{i = 1}^N$, define $\rmK_\rmX \in \mathbb{R}^{N \times N}$ to be $[\rmK_\rmX]_{i,j} = \rmK(\vx_i,\vx_j)$.
The prediction for an unseen data $\vx'$ is
$\sum_{i = 1}^N\alpha_i \rmK(\vx',\vect{x}_i)$ where $\valpha =  \rmK_\rmX^{-1} \vy.$

\subsection{Data Augmentation Schemes}
\label{sec:da}
In this paper we consider two types of data augmentation schemes: translation and horizontal flip.

\paragraph{Translation.}
Given $(i,j) \in [\nnw] \times [\nnh]$, we define the translation operator $\trans_{i, j}: \mathbb{R}^{\nnw \times \nnh \times \nnc} \rightarrow \mathbb{R}^{\nnw \times \nnh \times \nnc}$ as follow.
For an image $\vect{x} \in \mathbb{R}^{\nnw \times \nnh \times \nnc}$,
\[ \left[\trans_{i, j}\left(\vect{x}\right)\right]_{i',j', c} = \left[\vect{x}\right]_{i'+i,j'+j, c}\]
for $(i',j', c) \in [\nnw] \times [\nnh] \times [\nnc]$.
Here the precise definition of $ \left[\vect{x}\right]_{i'+i,j'+j,c}$ depends on the padding scheme.
Given a dataset $D = \left\{\left(\vect{x}_i,y_i\right)\right\}_{i=1}^N$, the \emph{full translation data augmentation scheme} creates a new dataset $D_{\trans} = \left\{\left( \trans_{i, j}\left(\vect{x}_i\right),y_i \right) \right\}_{(i,j,n) \in [\nnw] \times [\nnh] \times [N]}$ and training is performed on $D_{\trans}$.

\paragraph{Horizontal Flip.}
For an image $\vect{x} \in \mathbb{R}^{\nnw \times \nnh \times \nnc}$, the flip operator $\flip: \mathbb{R}^{\nnw \times \nnh \times \nnc} \rightarrow \mathbb{R}^{\nnw \times \nnh \times \nnc}$ is defined to be
\[\left[\flip\left(\vect{x}\right)\right]_{i,j,c} = \left[\vect{x}\right]_{\nnw + 1 -i,j,c}\]
for $(i,j,c) \in [\nnw] \times [\nnh] \times [\nnc]$.
Given a dataset $D = \left\{\left(\vect{x}_i,y_i\right)\right\}_{i=1}^N$, the {\em horizontal flip data augmentation scheme} creates a new dataset of the form $D_{\flip} = \left\{\left(\flip\left(\vect{x}_i\right),y_i\right)\right\}_{i=1}^N$ and training is performed on $D_{\flip} \cup D$.


\section{Equivalence Between Augmented Kernel and Data Augmentation}
\label{sec:gap}
In this section, we demonstrate the equivalence between data augmentation and augmented kernels.
To formally discuss the equivalence, we use group theory to describe translation  and horizontal flip operators.
We provide the definition of group in Section~\ref{sec:proofs} for completeness.

It is easy to verify that  $\{\mathcal{F}, \mathcal{I}\}$, $\{\mathcal{T}_{i,j}\}_{(i,j) \in [\nnw] \times [\nnh]}$,
$\{\mathcal{T}_{i,j}\circ \mathcal{F}\}_{(i,j) \in [\nnw] \times [\nnh]}\cup \{\mathcal{T}_{i,j}\}_{_{(i,j) \in [\nnw] \times [\nnh]}}$ are  groups, where $\mathcal{I}$ is the identity map.
From now on, given a dataset $(\rmX, \vect{y})$ with data $\left\{\left(\vect{x}_i,y_i\right)\right\}_{i=1}^N$ and a group $\mathcal{G}$, the augmented dataset $(\rmX_\gG, \vect{y}_\gG)$ is defined to be $\{g(\vect{x}_i), y_i\}_{g \in \mathcal{G}, i \in [N]}$.
The prediction for an unseen data $\vx'$ on the augmented dataset is $\sum_{i \in [N], g\in \gG}\talpha_{i, g} \rmK(\vx',g(\vect{x}_i))$ where $\vtalpha =  \left(\rmK_{\rmX_\gG}\right)^{-1} \vy_\gG.$

To proceed, we define the concept of {\em augmented kernel}.
Let $\gG$ be a finite group. Define the augmented kernel  $\rmK^{\gG}$ as \[\rmK^{\gG}(\vx,\vx') = \E_{g\in\gG} \rmK(g(\vx),\vx')\] where $\vx,\vx'$ are two inputs images and $g$ is drawn from $\gG$ uniformly at random.
A key observation is that for \cntk~and \cnngp, when circular padding and \gap~is adopted, the corresponding kernel is the augmented kernel of the group $\mathcal{G} =  \{\mathcal{T}_{i,j}\}_{(i,j) \in [\nnw] \times [\nnh]}$.
Formally, we have 
\[
\mat{\Sigma}_{\gap}\left(\vect{x},\vect{x}'\right) = \frac{1}{PQ}\mat{\Sigma}_{\fc}^{\gG}\left(\vect{x},\vect{x}'\right)
\] 
and 
\[
\mat{\Theta}_{\gap}\left(\vect{x},\vect{x}'\right) = \frac{1}{PQ} \mat{\Theta}_{\fc}^{\gG}\left(\vect{x},\vect{x}'\right),
\] which can be seen by checking the formula of these kernels and using definition of circular padding. 
Similarly, the following equivariance property holds for $\mat{\Sigma}_{\gap}, \mat{\Sigma}_{\fc}, \mat{\Theta}_{\gap}$ and $\mat{\Theta}_{\fc}$, under all groups mentioned above, including $\{\mathcal{F}, \mathcal{I}\}$ and $\{\mathcal{T}_{i,j}\}_{(i,j) \in [\nnw] \times [\nnh]}$.

\begin{defn}\label{def:inv}
    A kernel $\rmK$ is {\em equivariant} under a group $\gG$ if and only if for any $g\in \gG$, $\rmK(g(\vx), g(\vx')) = \rmK(\vx, \vx')$. 
\end{defn}
The following theorem formally states the equivalence between using an augmented kernel on the dataset and using the kernel on the augmented dataset.

\begin{thm}
\label{thm:gap_equivalence}
Given a group $\gG$ and a kernel $\rmK$ such that $\rmK$ is equivariant under $\gG$, then the prediction of augmented kernel $\rmK^\gG$ with dataset $(\rmX,\vy)$ is equal to that of kernel $\rmK$ and augmented dataset $(\rmX_{\gG},\vy_{\gG})$.
Namely, for any $\vx'\in \mathbb{R}^{P\times Q\times C}$, $\sum_{i= 1}^N\alpha_i \rmK^{\gG}(\vx',\vx_i)= \sum_{i \in [N], g\in \gG}\talpha_{i, g} \rmK(\vx',g(\vx_i))$ where $\valpha =  \left({\rmK_\rmX^{\gG}}\right)^{-1} \vy, \vtalpha =  \left(\rmK_{\rmX_\gG}\right)^{-1} \vy_\gG$.
\end{thm}

The proof is deferred to Appendix~\ref{sec:proofs}.
Theorem~\ref{thm:gap_equivalence} implies the following two corollaries. 
\begin{cor}\label{cor:gap}
  For $\gG = \{\mathcal{T}_{i,j}\}_{(i,j) \in [\nnw] \times [\nnh]}$, for any given dataset $D$, the prediction of $\mat\Sigma_\gap$ (or $\mat\Theta_\gap$) with dataset $D$ is equal to the prediction of $\mat\Sigma_\fc$ (or $\mat\Theta_\fc$) with augmented dataset $D_\mathcal T$.
\end{cor}

\begin{cor}\label{cor:flip}
  For $\gG = \{\mathcal F, \mathcal I\}$, for any given dataset $D$, the prediction of $\mat\Sigma_\gap^\gG$ (or $\mat\Theta_\gap^\gG$) with dataset $D$ is equal to the prediction of $\mat\Sigma_\gap$ (or $\mat\Theta_\gap$) with augmented dataset $D_\mathcal F\cup D$.
\end{cor}
Now we discuss implications of Theorem~\ref{thm:gap_equivalence} and its corollaries.
Naively applying data augmentation, with full translation on \cntk~or \cnngp~for example, one needs to create a much larger kernel matrix since there are $\nnw \nnh$ translation operators, which is often computationally infeasible.
Instead, one can directly use the augmented kernel ($\mat\Sigma_\gap$ or $\mat\Theta_\gap$ for the case of full translation on \cntk~or \cnngp) for prediction, for which one only needs to create a kernel matrix that is as large as the original one. For horizontal flip, although the augmentation kernel can not be conveniently computed as full translation, Corollary~\ref{cor:flip} still provides a more efficient method for computing kernel values and solving kernel regression, since the augmented dataset is twice as large as the original dataset, while the kernel matrix of the augmented kernel is as large as the original one.

\section{Local Average Pooling}
\label{sec:lap}

\begin{figure}[t]
    \centering
    \begin{subfigure}[t]{0.35\textwidth}
        \centering
        \includegraphics[width=0.8\textwidth]{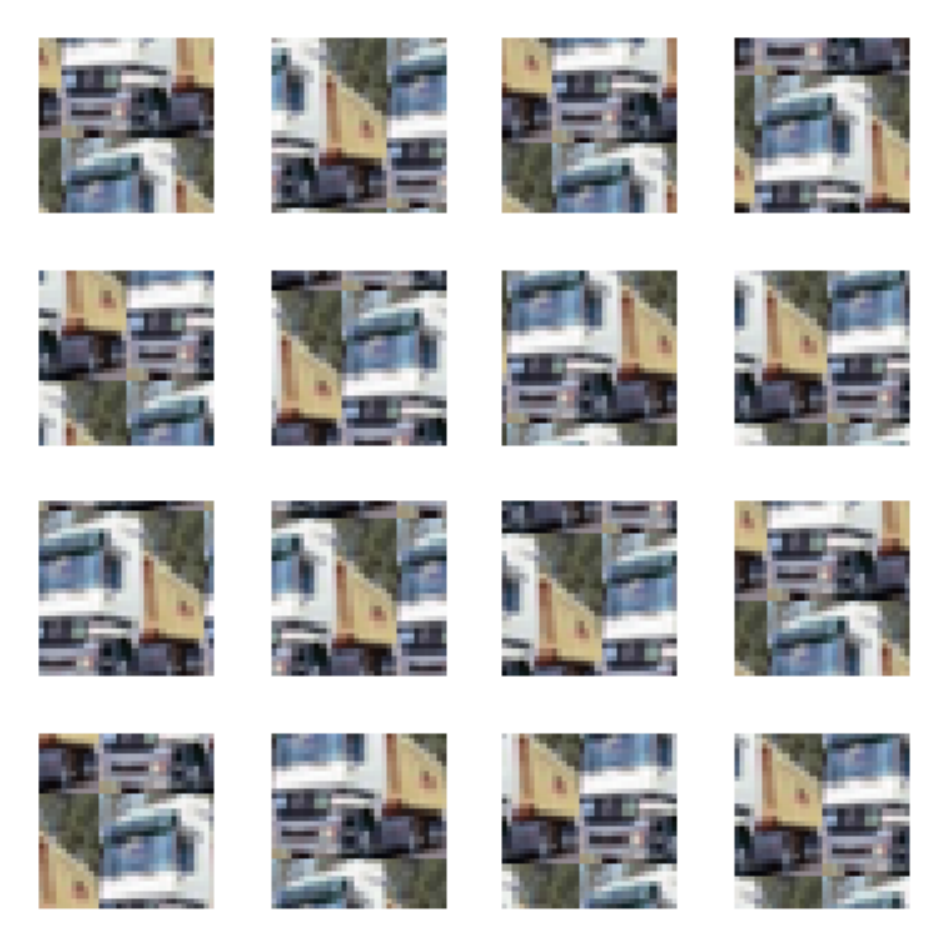}
        \caption{\gap}
    \end{subfigure}%
    ~
    \begin{subfigure}[t]{0.35\textwidth}
        \centering
        \includegraphics[width=0.8\textwidth]{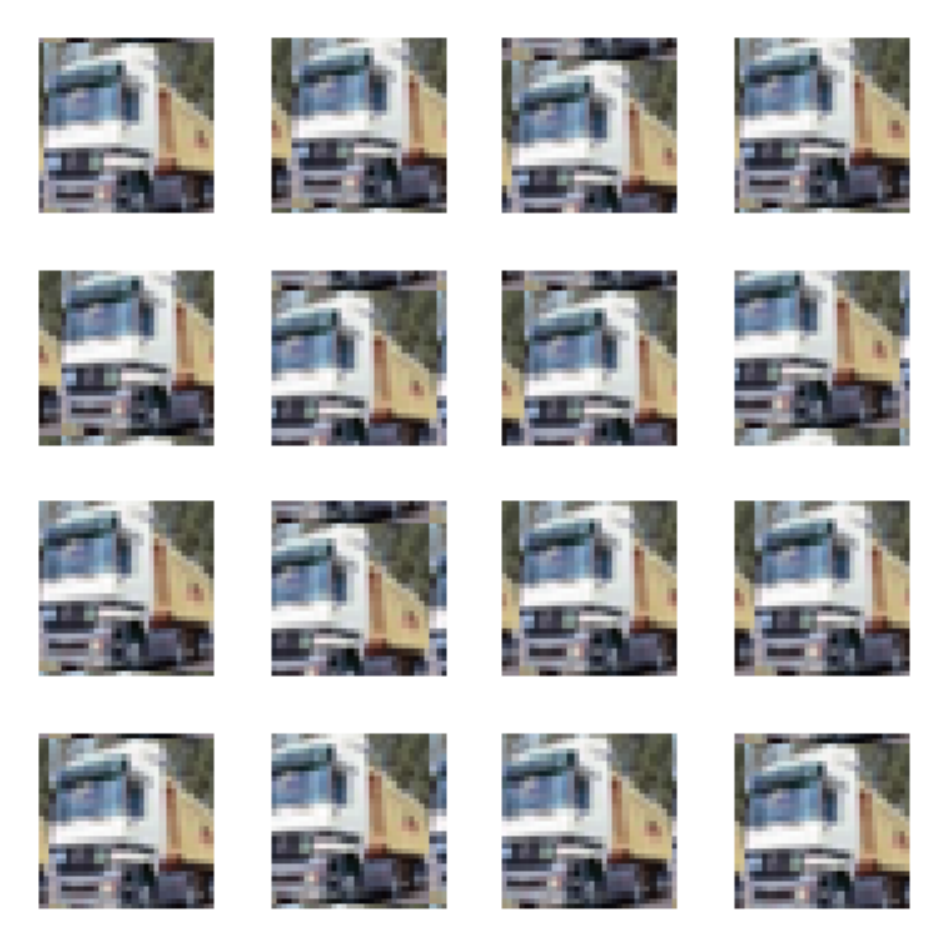}
        \caption{\lap~with $\lapc = 4$}
    \end{subfigure}
    \caption{Randomly sampled images with full translation data augmentation and local translation data augmentation from CIFAR-10.
    Full translation data augmentation can create unrealistic images that harm the performance whereas local translation data augmentation creates more realistic images.
    	}
    \label{fig:gap_lap}
\end{figure}
In this section, we introduce a new operation called {\em Local Average Pooling} (\lap).
As discussed in the introduction, full translation data augmentation may create unrealistic images.
A natural idea is to do local translation data augmentation, i.e., restricting the distance of translation.
More specifically, we only allow translation operations $\trans_{\Delta_i,\Delta_j}$ (cf. Section~\ref{sec:da}) for $(\Delta_i,\Delta_j)\in [-\lapc,\lapc] \times [-\lapc,\lapc]$ where $c$ is a parameter to control the amount of allowed translation.
With a proper choice of the parameter $c$, translation data augmentation will not create unrealistic images (cf. Figure~\ref{fig:gap_lap}).
However, naive local translation data augmentation is computationally infeasible for kernel methods, even for moderate choice of $c$.
To remedy this issue, in this section we introduce \lap, which is inspired by the connection between full translation data augmentation and \gap~on \cnngp~and \cntk.
Here, for simplicity, we assume $P = Q$ and derive the formula only for \cntk.
Our formula can be generalized to \cnngp~in a straightforward manner.


Recall that for two given images $\vect{x}$ and $\vect{x}'$, without \gap, the formula for output of \cntk~is $\trace{\mat{\Theta}(\vect{x},\vect{x}')}$.
With \gap, the formula for output of \cntk~is
\[\frac{1}{P^4}\sum_{i,j,i',j' \in [P]^4}\left[\mat{\Theta} \left(\vect{x},\vect{x}'\right)\right]_{i,j,i',j'}.\]
 With circular padding, the formula can be rewritten as 
 \[ \frac{1}{P^2} \expect_{\Delta_i, \Delta_{i}', \Delta_j, \Delta_{j}' \sim [P]^4} \sum_{i, j \in [P] \times [P]} \left[\mat{\Theta} \left(\vect{x},\vect{x}'\right)\right]_{i + \Delta_i,j + \Delta_j,i + \Delta_i',j + \Delta_j'},\]
 which is again equal to
 \[
   \frac{1}{P^2}\expect_{\Delta_i, \Delta_{i}', \Delta_j, \Delta_{j}' \sim [P]^4}
  \trace{\mat{\Theta}\left(\trans_{\Delta_i,\Delta_j}(\vect{x}),\trans_{\Delta_i',\Delta_j'}(\vect{x'})\right)}.
  \]
We ignore the $1 / P^2$ scaling factor since it plays no role in kernel regression.

Now we consider restricted translation operations $\trans_{\Delta_i,\Delta_j}$ with $(\Delta_i,\Delta_j)\in [-\lapc,\lapc] \times [-\lapc,\lapc]$ and derive the formula for \lap.
Assuming circular padding, we have
%
\begin{align}
& \expect\nolimits_{\Delta_i, \Delta_{i}', \Delta_j, \Delta_{j}' \sim [-\lapc,\lapc]^4}
  \trace{\mat{\Theta}\left(\trans_{\Delta_i,\Delta_j}(\vect{x}),\trans_{\Delta_i',\Delta_j'}(\vect{x'})\right)} \notag \\
 =~& \frac{1}{(2c+ 1)^4}\sum\nolimits_{\Delta_i, \Delta_{i}', \Delta_j, \Delta_{j}' \in[-\lapc,\lapc]^4}\sum\nolimits_{i, j\in [P]^2} \left[\mat{\Theta}(\vect{x},\vect{x'})\right]_{i + \Delta_i,j + \Delta_j,i + \Delta_i',j + \Delta_j'} \label{equ:lap}.
\end{align}

Now we have derived the formula for \lap, which is the RHS of Equation~\ref{equ:lap}.
Notice that the formula in the RHS of Equation~\ref{equ:lap} is a well-defined quantity for all padding schemes.
In particular, assuming zero padding, when $c = P$, \lap~is equivalent to \gap.
When $c = 0$, \lap~is equivalent to no pooling layer.
Another advantage of \lap~is that it does not incur any extra computational cost, since the formula in Equation~\ref{equ:lap} can be rewritten as
\[\sum_{i, j, i', j' \in [P]^4}  [\mat{w}]_{i, j, i', j'} \cdot \left[\mat{\Theta}(\vect{x},\vect{x'})\right]_{i, j, i', j'}\] where each entry in the weight tensor $\mat{w}$ can be calculated in constant time.

Note that the \gap~operation in \cnngp~and \cntk~corresponds to the \gap~layer~in \cnn s.
Here we observe that the following {\em box blur layer} corresponds to \lap~in \cnn s. Box blur layer ($\boxblur$) is a function $\R^{P\times Q} \to \R^{P\times Q}$ such that
\[[\boxblur(\vect{x})]_{i,j} = \frac{1}{(2c + 1)^2}\sum_{\Delta_i, \Delta_j \in [-\lapc,\lapc]^2} \vect{x}_{i + \Delta_i,j  + \Delta_j}.\] This is in fact the standard average pooling layer with pooling size $2c + 1$ and stride $1$. We prove the equivalence between \lap~and box blur layer in Appendix~\ref{sec:bblur}. 
In Section~\ref{sec:cnn_exp}, we verify the effectiveness of $\boxblur$~on \cnn s via experiments.

\section{Experiments}
\label{sec:exp}
In this section we present our empirical findings on CIFAR-10~\citep{krizhevsky2009learning} and Fashion-MNIST~\citep{xiao2017}.

\paragraph{Experimental Setup.}
For both CIFAR-10 and Fashion-MNIST we use the full training set and report the test accuracy on the full test set.
Throughout this section we only consider $3 \times 3$  convolutional filters with stride $1$ and no dilation.
In the convolutional layers in \cntk~and \cnngp, we use zero padding with pad size $1$ to ensure the input of each layer has the same size.
We use zero padding for \lap~throughout the experiment.
We perform standard preprocessing (mean subtraction and standard deviation division) for all images.

In all experiments, we perform kernel ridge regression to utilize the calculated kernel values\footnote{We also tried kernel SVM but found it significantly degrading the performance, and thus do not include the results.}.
We normalize the kernel matrices so that all diagonal entries are ones.
Equivalently, we ensure all features have unit norm in RKHS.
Since the resulting kernel matrices are usually ill-conditioned, we set the regularization term $\lambda = 5 \times 10^{-5}$, to make inverting kernel matrices numerically stable.
We use one-hot encodings of the labels as regression targets.
We use \texttt{scipy.linalg.solve} to solve the corresponding kernel ridge regression problem.

The kernel value of \cntk~and \cnngp~are calculated using the CuPy package.
We write native CUDA codes to speed up the calculation of the kernel values.
All experiments are performed on Amazon Web Services (AWS), using (possibly multiple) NVIDIA Tesla V100 GPUs.
For efficiency considerations, all kernel values are computed with 32-bit precision.

One unique advantage of the dynamic programming algorithm for calculating \cntk~and \cnngp~is that we do not need repeat experiments for, say, different values of $c$ in \lap~and different depths.
With our highly-optimized native CUDA codes, we spend roughly 1,000 GPU hours on calculating all kernel values for each dataset.

\subsection{Ablation Study on CIFAR-10 and Fashion-MNIST}\label{sec:exp_ablation}
We perform experiments to study the effect of different values of the $c$ parameter in \lap~and horizontal flip data argumentation on \cntk~and \cnngp.
For experiments in this section we set the bias term in \cntk~and \cnngp~to be $\gamma = 0$ (cf. Section~\ref{sec:cnngp_cntk_def}).
We use the same architecture for \cntk~and \cnngp~as in \cite{arora2019exact}.
I.e., we stack multiple convolutional layers before the final pooling layer.
We use $d$ to denote the number of convolutions layers, and in our experiments we set $d$ to be $5$, $8$, $11$ or $14$, to study the effect of depth on \cntk~and \cnngp.
For CIFAR-10, we set the $c$ parameter in \lap~to be $0, 4, \ldots, 32$, while for Fashion-MNIST we set the $c$ parameter in \lap~to be $0, 4, \ldots, 28$.
Notice that when $c = 32$ for CIFAR-10 or $c = 28$ for Fashion-MNIST, \lap~is equivalent to \gap, and when $c = 0$, \lap~is equivalent to no pooling layer.
Results on CIFAR-10 are reported in Tables~\ref{tab:cifar_cntk} and~\ref{tab:cifar_cnngp}, and results on Fashion-MNIST are reported in Tables~\ref{tab:fmnist_cntk} and~\ref{tab:fmnist_cnngp}.
In each table, for each combination of $c$ and $d$, the first number is the test accuracy without horizontal flip data augmentation (in percentage), and the second number (in parentheses) is the test accuracy with horizontal flip data augmentation.

\begin{table}
\centering
\begin{tabular}{|c|cccc|}
\hline
\diagbox{$c$}{$d$} & $5$ & $8$ & $11$ & $14$\\
\hline
0 &  $66.55$ $(69.87)$  &  $66.27$ $(69.87)$  &  $65.85$ $(69.37)$  &  $65.47$ $(68.90)$ \\
4 &  $77.06$ $(79.08)$  &  $77.14$ $(78.96)$  &  $77.06$ $(78.98)$  &  $76.52$ $(78.74)$ \\
8 &  $79.24$ $(80.95)$  &  $79.25$ $(81.03)$  &  $78.98$ $(80.94)$  &  $78.65$ $(80.35)$ \\
12 &  $\mathbf{80.11}$ $(81.34)$  &  $79.79$ $(81.28)$  &  $79.29$ $(81.14)$  &  $79.13$ $(80.91)$ \\
16 &  $79.80$ $(81.21)$  &  $79.71$ $(\mathbf{81.40})$  &  $79.74$ $(81.09)$  &  $79.42$ $(81.00)$ \\
20 &  $79.24$ $(80.67)$  &  $79.27$ $(80.88)$  &  $79.30$ $(80.76)$  &  $78.92$ $(80.39)$ \\
24 &  $78.07$ $(79.88)$  &  $78.16$ $(79.79)$  &  $78.14$ $(80.06)$  &  $77.87$ $(80.07)$ \\
28 &  $76.91$ $(78.69)$  &  $77.33$ $(79.20)$  &  $77.65$ $(79.56)$  &  $77.65$ $(79.74)$ \\
32 &  $76.79$ $(78.53)$  &  $77.39$ $(79.13)$  &  $77.63$ $(79.51)$  &  $77.63$ $(79.74)$ \\
\hline
\end{tabular}
\caption{Test accuracy of \cntk~on CIFAR-10.}
\label{tab:cifar_cntk}
\end{table}

\begin{table}
\centering
\begin{tabular}{|c|cccc|}
\hline
\diagbox{$c$}{$d$} & $5$ & $8$ & $11$ & $14$\\
\hline
0 &  $63.53$ $(67.90)$  &  $65.54$ $(69.43)$  &  $66.42$ $(70.30)$  &  $66.81$ $(70.48)$ \\
4 &  $76.35$ $(78.79)$  &  $77.03$ $(79.30)$  &  $77.39$ $(79.52)$  &  $77.35$ $(79.65)$ \\
8 &  $79.48$ $(81.32)$  &  $79.82$ $(81.49)$  &  $79.76$ $(81.71)$  &  $79.69$ $(81.53)$ \\
12 &  $80.40$ $(82.13)$  &  $80.64$ $(82.09)$  &  $80.58$ $(82.06)$  &  $80.32$ $(81.95)$ \\
16 &  $80.36$ $(81.73)$  &  $\mathbf{80.78}$ $(\mathbf{82.20})$  &  $80.59$ $(82.06)$  &  $80.41$ $(81.83)$ \\
20 &  $79.87$ $(81.50)$  &  $80.15$ $(81.33)$  &  $79.87$ $(81.46)$  &  $79.98$ $(81.35)$ \\
24 &  $78.60$ $(79.98)$  &  $78.91$ $(80.48)$  &  $79.22$ $(80.53)$  &  $78.94$ $(80.46)$ \\
28 &  $77.18$ $(78.84)$  &  $78.03$ $(79.86)$  &  $78.45$ $(79.87)$  &  $78.48$ $(80.07)$ \\
32 &  $77.00$ $(78.49)$  &  $77.85$ $(79.65)$  &  $78.49$ $(80.04)$  &  $78.45$ $(80.01)$ \\
\hline
\end{tabular}
\caption{Test accuracy of \cnngp~on CIFAR-10.}
\label{tab:cifar_cnngp}
\end{table}

\begin{table}[H]
\centering
\begin{tabular}{|c|cccc|}
\hline
\diagbox{$c$}{$d$} & $5$ & $8$ & $11$ & $14$\\
\hline
0 &  $92.25$ $(92.56)$  &  $92.22$ $(92.51)$  &  $92.11$ $(92.29)$  &  $91.76$ $(92.17)$ \\
4 &  $\mathbf{93.76}$ $(\mathbf{94.07})$  &  $93.69$ $(93.86)$  &  $93.55$ $(93.74)$  &  $93.37$ $(93.58)$ \\
8 &  $93.72$ $(93.96)$  &  $93.67$ $(93.78)$  &  $93.50$ $(93.58)$  &  $93.32$ $(93.51)$ \\
12 &  $93.59$ $(93.80)$  &  $93.58$ $(93.70)$  &  $93.35$ $(93.44)$  &  $93.21$ $(93.40)$ \\
16 &  $93.50$ $(93.62)$  &  $93.42$ $(93.63)$  &  $93.27$ $(93.40)$  &  $93.10$ $(93.25)$ \\
20 &  $93.10$ $(93.34)$  &  $93.17$ $(93.49)$  &  $93.20$ $(93.34)$  &  $92.99$ $(93.18)$ \\
24 &  $92.77$ $(93.04)$  &  $93.07$ $(93.44)$  &  $93.11$ $(93.31)$  &  $93.02$ $(93.21)$ \\
28 &  $92.80$ $(92.98)$  &  $93.08$ $(93.42)$  &  $93.12$ $(93.28)$  &  $92.97$ $(93.19)$ \\
\hline
\end{tabular}
\caption{Test accuracy of \cntk~on Fashion-MNIST.}
\label{tab:fmnist_cntk}
\end{table}

\begin{table}[H]
\centering
\begin{tabular}{|c|cccc|}
\hline
\diagbox{$c$}{$d$} & $5$ & $8$ & $11$ & $14$\\
\hline
0 &  $91.47$ $(91.81)$  &  $91.96$ $(92.37)$  &  $92.09$ $(92.60)$  &  $92.22$ $(92.72)$ \\
4 &  $93.44$ $(93.60)$  &  $93.59$ $(\mathbf{93.79})$  &  $\mathbf{93.63}$ $(93.76)$  &  $93.59$ $(93.64)$ \\
8 &  $93.26$ $(93.16)$  &  $93.41$ $(93.51)$  &  $93.31$ $(93.52)$  &  $93.39$ $(93.46)$ \\
12 &  $92.83$ $(92.94)$  &  $93.07$ $(93.20)$  &  $93.11$ $(93.15)$  &  $92.94$ $(93.09)$ \\
16 &  $92.46$ $(92.51)$  &  $92.58$ $(92.83)$  &  $92.64$ $(92.92)$  &  $92.68$ $(93.07)$ \\
20 &  $91.83$ $(91.72)$  &  $92.35$ $(92.42)$  &  $92.49$ $(92.79)$  &  $92.51$ $(92.69)$ \\
24 &  $91.15$ $(91.40)$  &  $92.10$ $(92.18)$  &  $92.29$ $(92.60)$  &  $92.41$ $(92.77)$ \\
28 &  $91.30$ $(91.37)$  &  $92.03$ $(92.27)$  &  $92.41$ $(92.79)$  &  $92.41$ $(92.74)$ \\
\hline
\end{tabular}
\caption{Test accuracy of \cnngp~on Fashion-MNIST.}
\label{tab:fmnist_cnngp}
\end{table}

We made the following observations regarding our experimental results.

\begin{itemize}
\item \lap~with a proper choice of the parameter $c$ significantly improves the performance of \cntk~and \cnngp.
On CIFAR-10, the best-performing value of $c$ is $c = 12$ or $16$, while on Fashion-MNIST the best-performing value of $c$ is $c = 4$.
We suspect this difference is due to the nature of the two datasets: CIFAR-10 contains real-life images and thus allow more translation, while Fashion-MNIST contains images with centered clothes and thus allow less translation.
For both datasets, the best-performing value of $c$ is consistent across all settings (depth, \cntk~or \cnngp) that we have considered.
\item Horizontal flip data augmentation is less effective on Fashion-MNIST than on CIFAR-10.
There are two possible explanations for this phenomenon. First, most images in Fashion-MNIST are nearly horizontally symmetric (e.g., T-shirts and bags).
Second, \cntk~and \cnngp~have already achieved a relatively high accuracy on Fashion-MNIST, and thus it is reasonable for horizontal flip data augmentation to be less effective on this dataset.
\item Finally, for \cntk, when $c = 0$ (no pooling layer) and $c = 32$ (\gap) our reported test accuracies are close to those in~\cite{arora2019exact} on CIFAR-10. For \cnngp, when $c = 0$ (no pooling layer) our reported test accuracies are close to those in~\cite{novak2019bayesian} on CIFAR-10 and Fashion-MNIST.
This suggests that we have reproduced previous reported results.
\end{itemize}

\subsection{Improving Performance on CIFAR-10 via Additional Pre-processing}\label{sec:exp_cn}
Finally, we explore another interesting question: what is the limit of non-deep-neural-network methods on CIFAR-10?
To further improve the performance, we combine \cntk~and \cnngp~with \lap, together with the previous best-performing non-deep-neural-network method~\cite{coates2011analysis}.
Here we use the variant implemented in~\cite{recht2019imagenet}\footnote{\url{https://github.com/modestyachts/nondeep}}.
More specifically, we first sample {\em 2048} random image patches with size $5 \times 5$ from all training images.
Then for the sampled images patches, we subtract the mean of the patches, then normalize them to have unit norm, and finally perform ZCA transformation to the resulting patches.
We use the resulting patches as 2048 filters of a convolutional layer with kernel size $5$, stride $1$ and no dilation or padding.
For an input image $\vx$, we use $\mathtt{conv}(\vx)$ to denote the output of the convolutional layer.
As in the implementation in~\cite{recht2019imagenet}, we use $\mathrm{ReLU}(\mathtt{conv}(\vx) - \gamma_{\mathrm{feature}})$ and $\mathrm{ReLU}(-\mathtt{conv}(\vx) - \gamma_{\mathrm{feature}})$ as the input feature for \cntk~and~\cnngp.
Here we fix $\gamma_{\mathrm{feature}} = 1$ as in~\cite{recht2019imagenet}, and set the bias term $\gamma$ in \cntk~and \cnngp~to be $\gamma = 3$, which is the filter size used in \cntk~and~\cnngp.
To make the equivariant  under horizontal flip (cf. Defintion~\ref{def:inv}), for each image patch, we horizontally flip it and add the flipped patch into the convolutional layer as a new filter.
Thus, for an input CIFAR-10 image of size $32 \times 32$, the dimension of the output feature is $8192 \times 28 \times 28$.
To isolate the effect of randomness in the choices of the image patches, we fix the random seed to be 0 throughout the experiment.
In this experiment, we set the value of the $c$ parameter in \lap~to be $4, 8, 12, \ldots, 20$ to avoid small and large values of $c$.
The results are reported in Tables~\ref{tab:cn_cntk} and~\ref{tab:cn_cnngp}.
In each table, for each combination of $c$ and $d$, the first number is the test accuracy without horizontal flip data augmentation (in percentage), and the second number (in parentheses) is the test accuracy with horizontal flip data augmentation (again in percentage).

\begin{table}
\centering
\begin{tabular}{|c|cccc|}
\hline
\diagbox{$c$}{$d$} & $5$ & $8$ & $11$ & $14$\\
\hline
4 &  $84.63$ $(86.64)$  &  $84.07$ $(86.23)$  &  $83.29$ $(85.53)$  &  $82.57$ $(84.81)$ \\
8 &  $86.36$ $(88.32)$  &  $85.80$ $(87.81)$  &  $85.01$ $(87.08)$  &  $84.57$ $(86.53)$ \\
12 &  $86.74$ $(88.35)$  &  $86.20$ $(87.90)$  &  $85.60$ $(87.36)$  &  $84.95$ $(86.99)$ \\
16 &  $\mathbf{86.77}$ $(\mathbf{88.36})$  &  $86.17$ $(87.85)$  &  $85.60$ $(87.44)$  &  $84.92$ $(86.98)$ \\
20 &  $86.17$ $(87.77)$  &  $85.71$ $(87.50)$  &  $85.14$ $(87.07)$  &  $84.59$ $(86.84)$ \\
\hline
\end{tabular}
\caption{Test accuracy of additional pre-processing + \cntk~on CIFAR-10.}
\label{tab:cn_cntk}
\end{table}

\begin{table}
\centering
\begin{tabular}{|c|cccc|}
\hline
\diagbox{$c$}{$d$} & $5$ & $8$ & $11$ & $14$\\
\hline
4 &  $85.49$ $(87.32)$  &  $85.37$ $(87.22)$  &  $85.16$ $(87.11)$  &  $84.79$ $(86.81)$ \\
8 &  $87.07$ $(88.64)$  &  $86.82$ $(88.68)$  &  $86.53$ $(88.40)$  &  $86.39$ $(88.15)$ \\
12 &  $87.23$ $(88.91)$  &  $87.12$ $(\mathbf{88.92})$  &  $86.87$ $(88.66)$  &  $86.62$ $(88.29)$ \\
16 &  $\mathbf{87.28}$ $(88.90)$  &  $87.11$ $(88.66)$  &  $86.92$ $(88.61)$  &  $86.74$ $(88.24)$ \\
20 &  $86.81$ $(88.26)$  &  $86.77$ $(88.24)$  &  $86.61$ $(88.14)$  &  $86.26$ $(87.84)$ \\
\hline
\end{tabular}
\caption{Test accuracy of additional pre-processing + \cnngp~on CIFAR-10.}
\label{tab:cn_cnngp}
\end{table}

From our experimental results, it is evident that combining \cntk~or \cnngp~with additional pre-processing can significantly improve upon the performance of using solely \cntk~or~\cnngp, and that of using solely the approach in~\cite{coates2011analysis}.
Previously, it has been reported in~\cite{recht2019imagenet} that using solely the approach in ~\cite{coates2011analysis} (together with appropriate pooling layer) can only achieve a test accuracy of 84.2\% using {\em 256, 000} image patches, or 83.3\% using {\em 32, 000} image patches.
Even with the help of horizontal flip data augmentation, the approach in~\cite{coates2011analysis} can only achieve a test accuracy of 85.6\% using {\em 256, 000} image patches, or 85.0\% using {\em 32, 000} image patches.
Here we use significantly less image patches (only {\em 2048}) but achieve a much better performance, with the help of \cntk~and \cnngp.
In particular, we achieve a performance of 88.92\% on CIFAR-10, matching the performance of AlexNet on the same dataset.
In the setting reported in~\cite{coates2011analysis}, increasing the number of sampled image patches will further improve the performance.
Here we also conjecture that in our setting, further increasing the number of sampled image patches can improve the performance and get close to modern \cnn s.
However, due the limitation on computational resources, we leave exploring the effect of number of sampled image patches as a future research direction.

\subsection{Experiments on \cnn~with $\boxblur$}\label{sec:exp_cnn}
\label{sec:cnn_exp}
In Figure~\ref{fig:concave_curve}, we verify the effectiveness of $\boxblur$ on a 10-layer \cnn~(with Batch Normalization) on CIFAR-10.
The setting of this experiment is reported in Appendix~\ref{sec:exp_cnn_setting}.
Our network structure has no pooling layer except for the $\boxblur$ layer before the final fully-connected layer.
The fully-connected layer is fixed during the training.
Our experiment illustrates that even with a fixed final FC layer, using \gap~could improve the performance of \cnn, and challenges the conjecture that \gap~reduces the number of parameters in the last fully-connected layer and thus avoids overfitting.
Our experiments also show that $\boxblur$ with appropriate choice of $c$ achieves better performance than \gap.

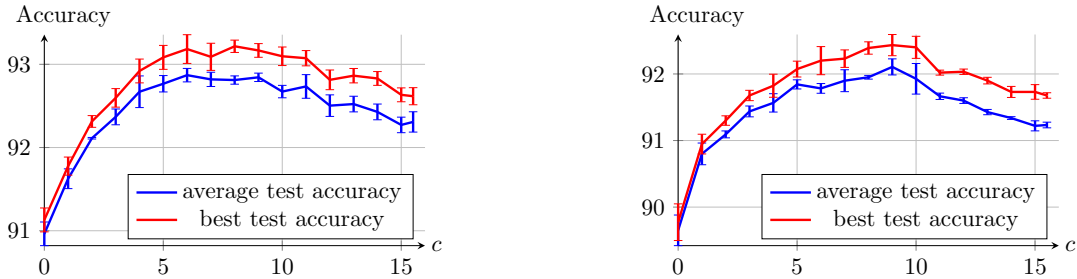
\begin{figure}
    \centering
    \begin{subfigure}[t]{0.5\textwidth}
        \centering
        \begin{tikzpicture}[thick,scale=0.9, every node/.style={scale=0.9}]
            \begin{axis}[
                  x label style={at={(axis description cs:1.03,0.06)},anchor=north},
                  y label style={at={(axis description cs:0.05,1)},rotate=270,anchor=south},
                      xlabel     = $c$, 
                      ylabel     = Accuracy, 
                      axis lines = left, 
                      clip       = false,
                      xmax=16,
                      x=10,
                      y=35,
                      grid=major,
                      legend pos=south east,
                    ]
            \addplot [color=blue,line width=1pt]
             plot [error bars/.cd, y dir = both, y explicit, error bar style={line width=1pt,solid}]
             table[y error index=2]{\jobname1.dat};
            \addplot [color=red,line width=1pt]
             plot [error bars/.cd, y dir = both, y explicit, error bar style={line width=1pt,solid}]
             table[y error index=2]{\jobname2.dat};
             \legend{average test accuracy, best test accuracy}
            \end{axis}
        \end{tikzpicture}
        \caption{With Horizontal Flip Data Augmentation}
    \end{subfigure}%
    ~
    \begin{subfigure}[t]{0.5\textwidth}
        \centering
        \begin{tikzpicture}[thick,scale=0.9, every node/.style={scale=0.9}]
        \begin{axis}[
                  x label style={at={(axis description cs:1.03,0.06)},anchor=north},
                  y label style={at={(axis description cs:0.05,1)},rotate=270,anchor=south},
                  xlabel     = $c$, 
                  ylabel     = Accuracy, 
                  axis lines = left, 
                  clip       = false,
                  xmax=16,
                  x=10,
                  y=28,
                  grid=major,
                  legend pos=south east,
                ]
        \addplot [color=blue,line width=1pt]
         plot [error bars/.cd, y dir = both, y explicit, error bar style={line width=1pt,solid}]
         table[y error index=2]{\jobname3.dat};
        \addplot [color=red,line width=1pt]
         plot [error bars/.cd, y dir = both, y explicit, error bar style={line width=1pt,solid}]
         table[y error index=2]{\jobname4.dat};
         \legend{average test accuracy, best test accuracy}
        \end{axis}
        \end{tikzpicture}
        \caption{Without Horizontal Flip Data Augmentation}
    \end{subfigure}
    \caption{
    Test accuracy of 10-layer \cnn~with various values for the $c$ parameter in $\boxblur$.}
    \label{fig:concave_curve}
\end{figure}

\section{Conclusion}
\label{sec:con}
In this paper, inspired by the connection between full translation data augmentation and \gap, we derive a new operation, \lap, on \cntk~and \cnngp, which consistently improves the performance on image classification tasks. Combining \cnngp~with \lap~and the pre-processing technique proposed by~\citet{coates2011analysis}, the resulting kernel achieves 89\% accuracy on CIFAR-10, matching the performance of AlexNet and is the strongest classifier that is not a trained neural network.

Here we list a few future research directions. Is that possible to combine more modern techniques on \cnn, such as batch normalization and residual layers, with \cntk~or \cnngp, to further improve the performance? Moreover, it is an interesting direction to study other components in modern \cnn s through the lens of \cntk~and \cnngp.

\section*{Acknowledgements}
S. Arora, W. Hu, Z. Li and D. Yu are supported by NSF, ONR, Simons Foundation, Schmidt Foundation, Amazon Research, DARPA and SRC.
S. S. Du is supported by National Science Foundation (Grant No. DMS-1638352) and the Infosys Membership.
R. Salakhutdinov and R. Wang are supported in part by NSF IIS-1763562, Office of Naval Research grant N000141812861, and Nvidia NVAIL award. 
Part of this work was done while R. Wang was visiting Princeton University.
The authors would like to thank Amazon Web Services for providing compute time for the experiments in this paper.

\bibliography{simonduref}
\bibliographystyle{plainnat}
\newpage
\appendix
\section{Formal Definitions of \cnn, \cnngp~and \cntk}
\label{sec:cnngp_cntk_def}
In this section we use the following additional notations.
Let $\mat I$ be the identity matrix, and $[n]=\{1, 2, \ldots, n\}$.
Let $\vect{e}_{i}$ be an indicator vector with $i$-th entry being $1$ and other entries being $0$, and let $\vect{1}$ denote the all-one vector.
We use $\odot$ to denote the pointwise product and $\otimes$ to denote the tensor product.
We use $\diag(\cdot)$ to transform a vector to a diagonal matrix.
We use $\relu{\cdot}$ to denote the activation function, such as the rectified linear unit (ReLU) function: $\relu{z} = \max\{z, 0\}$, and $\dot{\sigma}\left(\cdot\right)$ to denote the derivative of $\relu{\cdot}$.
Moreover, $c_{\sigma}$ is a fixed constant.
Denote by $\gauss(\bm\mu, \mat\Sigma)$ the Gaussian distribution with mean $\bm\mu$ and covariance $\mat{\Sigma}$.

We first define the convolution operation. 
For a convolutional filter $\vect{w} \in \mathbb{R}^{q \times q}$ and an image $\vect{x} \in \mathbb{R}^{\nnw \times \nnh}$, the convolution operator is defined as
\begin{align}
[\vect{w}\conv \vect{x}]_{i, j} = \sum_{a = -\frac{q-1}{2}}^{\frac{q-1}{2}} \sum_{b = -\frac{q-1}{2}}^{\frac{q-1}{2}}  \left[\mat{w}\right]_{a+\frac{q+1}{2},b+\frac{q+1}{2}} [\vect{x}]_{a+i,b+j}\text{ for }i \in [\nnw], j \in [\nnh]. \label{eqn:conv}
\end{align}

Here the precise definition of $\left[\mat{w}\right]_{a+\frac{q+1}{2},b+\frac{q+1}{2}}$ and $[\vect{x}]_{a+i,b+j}$ depends on the padding scheme (cf. Section~\ref{sec:cnngp_cntk}).
Notice that in Equation~\ref{eqn:conv}, the value of $[\vect{w}\conv \vect{x}]_{i, j}$ depends on $[\vect{x}]_{i-\frac{q-1}{2}:i+\frac{q-1}{2},  j-\frac{q-1}{2}: j+\frac{q-1}{2}}$.
Thus, for $(i,j,i',j') \in [\nnw]\times [\nnh] \times [\nnw] \times [\nnh]$, we define 
\[
\indset_{ij,i'j'} = \left\{(i+a,j+b,i'+a',j'+b') \in [\nnw]\times [\nnh] \times [\nnw] \times [\nnh] \mid -(q - 1) / 2 \le a, b, a', b' \le (q - 1) /2
\right\}.
\]

Now we formally define \cnn, \cnngp~and~\cntk.

\paragraph{\cnn.}
\begin{itemize}
	\item Let  $\vect{x}^{(0)} =\vect{x} \in \mathbb{R}^{\nnw\times \nnh \times \nnc^{(0)}}$ be the input image where $\nnc^{(0)}$ is the initial number of channels.
	\item For $h=1,\ldots,L$, $\beta = 1,\ldots,\nnc^{(h)}$, the intermediate outputs are defined as \begin{align*}
	\tilde{\vect{x}}_{(\beta)}^{(h)} = \sum_{\alpha=1}^{\nnc^{(h-1)}} \mat{W}_{(\alpha),(\beta)}^{(h)} \conv \vect{x}_{(\alpha)}^{(h-1)}  + \gamma \cdot b_{(\beta)} ,\quad
	\vect{x}^{(h)}_{(\beta)} = \sqrt{\frac{c_{\sigma}}{\nnc^{(h)} \times q \times q}}\act{\tilde{\vect{x}}_{(\beta)}^{(h)}}
	\end{align*}
	where each $\mat{W}_{(\alpha),(\beta)}^{(h)} \in \mathbb{R}^{q \times q}$ is a filter with Gaussian initialization, $b_{(\beta)}$ is a bias term with Gaussian initialization, and $\gamma$ is the scaling factor for the bias term.
\end{itemize}

\paragraph{\cnngp~and~\cntk.}
\begin{itemize}
	\item For $\alpha = 1,\ldots,\nnc^{(0)}$, $(i,j,i',j') \in [\nnw] \times [\nnh] \times [\nnw] \times [\nnh]$, define \begin{align*}
	\mat{K}^{(0)}_{(\alpha)}\left(\vect{x},\vect{x}'\right) = \vect{x}_{(\alpha)} \otimes \vect{x}_{(\alpha)}' \text{ and }\left[\mat{\Sigma}^{(0)}(\vect{x},\vect{x}')\right]_{ij,i'j'} = \sum_{\alpha=1}^{\nnc^{(0)}}\tr\left(\left[\mat{K}^{(0)}_{(\alpha)}(\vect{x},\vect{x}')\right]_{ \indset_{ij,i'j'}}\right) + {\gamma^2}.
	\end{align*}
	
	\item For $h \in [L - 1]$, \begin{itemize}
		\item For $(i,j,i',j') \in [\nnw] \times [\nnh] \times [\nnw] \times [\nnh]$, define \begin{align*}
		\mat{\twotwomat}_{ij,i'j'}^{(h)}(\vect{x},\vect{x}') = \begin{pmatrix}
		\left[\mat{\Sigma}^{(h-1)}(\vect{x},\vect{x})\right]_{ij,ij} & \left[\mat{\Sigma}^{(h-1)}(\vect{x},\vect{x}')\right]_{ij,i'j'} \\
		\left[\mat{\Sigma}^{(h-1)}\left(\vect{x}',\vect{x}\right)\right]_{i'j',ij}&
		\left[\mat{\Sigma}^{(h-1)}\left(\vect{x}',\vect{x}'\right)\right]_{i'j',i'j'}
		\end{pmatrix} \in \mathbb{R}^{2 \times 2}. 
		\end{align*}
		\item For  $(i,j,i',j') \in [\nnw] \times [\nnh] \times [\nnw] \times [\nnh]$, define
		\begin{align}
		\left[\mat{K}^{(h)}(\vect{x},\vect{x}')\right]_{ij,i'j'} = &\frac{c_{\sigma}}{q^2} \cdot  \expect_{(u,v)\sim \gauss\left(\vect{0},\mat{\twotwomat}_{ij,i'j'}^{(h)}(\vect{x},\vect{x}')\right)}\left[\act{u}\act{v}\right], \label{eqn:vanila_cnn_exp}\\
		\left[\dot{\mat{K}}^{(h)}(\vect{x},\vect{x}')\right]_{ij,i'j'} = &\frac{c_{\sigma}}{q^2} \cdot \expect_{(u,v)\sim \gauss\left(\vect{0},\mat{\twotwomat}_{ij,i'j'}^{(h)}(\vect{x},\vect{x}')\right)}\left[\deract{u}\deract{v}\right]. \label{eqn:vanila_cnn_exp_d}
		\end{align}
		\item For $(i,j,i',j') \in [\nnw] \times [\nnh] \times [\nnw] \times [\nnh]$, define 
		\begin{align*}
		\left[\mat{\Sigma}^{(h)}(\vect{x},\vect{x}')\right]_{ij,i'j'} = &\tr\left(\left[\mat{K}^{(h)}(\vect{x},\vect{x}')\right]_{D_{ij,i'j'}}\right) + {\gamma^2}. 
		\end{align*}
		\end{itemize}
		\end{itemize}
		Note that the definition of $\mat{\Sigma}(\vect{x},\vect{x}')$ and $\dot{\mat{\Sigma}}(\vect{x},\vect{x}')$ share similar patterns as their \ntk~counterparts~\citep{jacot2018neural}.
		The only difference is that we have one more step, taking the trace over patches.
		This step represents the convolution operation in the corresponding \cnn.
		Now we can define the kernel value recursively.
		\begin{enumerate}
			\item First, we define  $\mat{\Theta}^{(0)}(\vect{x},\vect{x}') = \mat{\Sigma}^{(0)}(\vect{x},\vect{x}')$.
			\item For $h\in[L - 1]$ and $(i,j,i',j') \in [\nnw] \times [\nnh] \times [\nnw] \times [\nnh]$, we define
			\begin{align*}
			\left[\mat{\Theta}^{(h)}(\vect{x},\vect{x}')\right]_{ij,i'j'} = \tr\left(\left[\dot{\mat{K}}^{(h)}(\vect{x},\vect{x}')\odot\mat{\Theta}^{(h-1)}(\vect{x},\vect{x}')+\mat{K}^{(h)}(\vect{x},\vect{x}')\right]_{D_{ij,i'j'}}\right) + {\gamma^2} .\label{eqn:vanila_cnn_gradient_kernel}
			\end{align*}
			\item Finally, define 
			\begin{align*}
			\mat{\Theta}^{(L)}(\vect{x},\vect{x}') =\dot{\mat{K}}^{(L)}(\vect{x},\vect{x}')\odot\mat{\Theta}^{(L-1)}(\vect{x},\vect{x}')+\mat{K}^{(L)}(\vect{x},\vect{x}').
			\end{align*}
			\end{enumerate}

\section{Additional Definitions and Proof of Theorem~\ref{thm:gap_equivalence}}
\label{sec:proofs}
\begin{defn}[Group of Operators]
	$(\mathcal{G}, \circ)$ is a \emph{group} of operators, if and only if
	\begin{enumerate}
		\item Each element $g\in \mathcal{G}$ is an operator: $\mathbb{R}^{P\times Q\times C}\to \mathbb{R}^{P\times Q\times C}$;
		\item $\forall g_1,g_2\in\mathcal{G}, g_1\circ g_2\in \mathcal{G}$, where $(g_1\circ g_2)(\vx)$ is defined as $g_1(g_2(\vx))$.
        \item $\exists e \in \mathcal G$, such that $\forall g \in \mathcal G$, $e\circ g = g\circ e = g$.
        \item $\forall g_1 \in \mathcal G$, $\exists g_2 \in \mathcal G$, such that $g_1 \circ g_2 = g_2 \circ g_1 = e$. We say $g_2$ is the inverse of $g_1$, namely, $g_2 = g_1^{-1}$.
	\end{enumerate}
\end{defn}

\begin{proof}[Proof of Theorem~\ref{thm:gap_equivalence}]
	Since we assume $\rmK^\gG_\rmX$ and $\rmK_{\rmX_\gG}$ are invertible, both $\valpha$ and $\vtalpha$ are uniquely defined. Now we claim $\vtalpha_g = \{\talpha_{i,g}\}_{i\in [N]}\in \mathbb{R}^N$  is equal to $\frac{\valpha}{|\gG|}$ for all $g\in\gG$.

By the equivariance of $\rmK$ under $\gG$, for all $j\in [N]$ and $g'\in \gG$,
\begin{align*}
\sum_{i\in[N], g\in\gG} \frac{\alpha_i}{|\gG|}\rmK(g'(\vx_j), g(\vx_i)) =&~ \sum_{i\in[N], g\in\gG} \frac{\alpha_i}{|\gG|}\rmK((g^{-1}\circ g')(\vx_j), \vx_i)\\
 =&~ \sum_{i\in[N]} \alpha_i\E_{g\in\gG}\rmK(g(\vx_j), \vx_i)\\
 =&~ \sum_{i\in [N]} \alpha_i\rmK^\gG(\vx_j, \vx_i) \\
 =&~ y_j.
\end{align*}
Note that $\vtalpha$ is defined as the unique solution of $\rmK_{\rmX_\gG} \vtalpha = \vy_\gG$.

Similarly, we have
\begin{align*}
 \sum_{i\in[N], g\in\gG} \frac{\alpha_i}{|\gG|}\rmK(\vx', g(\vx_i)) =
 \sum_{i\in[N]} \alpha_i\E_{g\in\gG}\rmK(g^{-1}(\vx'), \vx_i)
 = \sum_{i\in [N]} \alpha_i\rmK^\gG(\vx', \vx_i).
\end{align*}
\end{proof}

\section{Equivalence Between \lap~and Box Blur Layer.}
\label{sec:bblur}
For a \cnn~with a box blur layer before the final fully-connected layer, the final output is defined as $f(\params,\vect{x}) = \sum_{\alpha=1}^{\nnc^{(L)}} \left\langle \mat{W}_{(\alpha)}^{(L+1)},\boxblur\left(\vect{x}_{(\alpha)}^{(L)}\right)\right\rangle$,
	where $\vect{x}^{(L)}_{(\alpha)} \in \mathbb{R}^{\nnw \times \nnh}$, and $\mat{W}_{(\alpha)}^{(L+1)} \in \mathbb{R}^{\nnw \times \nnh}$ is the weight of the last fully-connected layer.
	
Now we establish the equivalence between $\boxblur$~and \lap~on \cntk.
The equivalence on \cnngp~can be derived similarly.
Let
 $\mat{\Theta}_{\boxblur}\left(\vect{x},\vect{x}'\right) \in \mathbb{R}^{[\nnw] \times [\nnh] \times [\nnw] \times [\nnh]}$ be the \cntk~kernel of $\boxblur\left(\vect{x}_{(\alpha)}^{(L)}\right)$.
 Since $\boxblur$~is just a linear operation, we have
\[
  \left[\mat{\Theta}_{\boxblur}\left(\vect{x},\vect{x}'\right)\right]_{i,j,i',j'}
  =  \frac{1}{(2c+ 1)^4} \sum_{\Delta_i,\Delta_j,\Delta_i',\Delta_j'\in[-\lapc,\lapc]^4}\left[\mat{\Theta}^{(L)}\left(\vect{x},\vect{x}'\right)\right]_{i + \Delta_i,j + \Delta_j,i' + \Delta_i',j'+ \Delta_j'}.\]
By the formula of the output kernel value for \cntk~without \gap,
we obtain \[\trace{\mat{\Theta}_{\boxblur}\left(\vect{x},\vect{x}'\right)} = \frac{1}{(2c+ 1)^4}\sum_{\Delta_i, \Delta_{i}', \Delta_j, \Delta_{j}' \in[-\lapc,\lapc]^4}\sum_{i, j\in [P] \times [Q]} \left[\mat{\Theta}(\vect{x},\vect{x'})\right]_{i + \Delta_i,j + \Delta_j,i + \Delta_i',j + \Delta_j'}.\]

\section{Setting of the Experiment in Section~\ref{sec:exp_cnn}}\label{sec:exp_cnn_setting}
The total number of training epochs is 80, and the learning rate is 0.1 initially, decayed by 10 at epoch 40 and 60 respectively.
The momentum is 0.9 and the weight decay factor is 0.0005.
In Figure~\ref{fig:concave_curve}, the blue line reports the average test accuracy of the last 10 epochs,
while the red line reports the best test accuracy of the total 80 epochs.
Each experiment is repeated for 3 times.
We use circular padding for both convolutional layers and the $\boxblur$ layer. 
The last data point with largest $x$-coordinate reported in Figure~\ref{fig:concave_curve} corresponds to \gap.

\end{document}